\documentclass{article}

\usepackage[final]{neurips_2019}

\usepackage{amsmath}
\usepackage[noend]{algpseudocode}
\usepackage{amssymb}
\usepackage{amsfonts}
\usepackage{fullpage}
\usepackage{float}
\usepackage{tikz}
\usetikzlibrary{bayesnet}
\usetikzlibrary{arrows}
\usepackage{color}
\usepackage{xcolor}
\usepackage{graphicx}
\usepackage{caption}
\usepackage{subcaption}
\usepackage{amsthm}
\usepackage{natbib}
\usepackage{bm}

\usepackage [english]{babel}
\usepackage [autostyle, english = american]{csquotes}
\MakeOuterQuote{"}

\usetikzlibrary{backgrounds}
\newcommand{\bw}{\mathbf{w}}
\newcommand{\by}{\mathbf{y}}
\newcommand{\bz}{\mathbf{z}}
\newcommand{\bxi}{\bm{\xi}}
\newcommand{\btheta}{\bm{\theta}}
\newtheorem{thm}{Theorem}
\definecolor{lasallegreen}{rgb}{0.03, 0.47, 0.19}
\definecolor{jonquil}{rgb}{0.98, 0.85, 0.37}
\usepackage[utf8]{inputenc} 
\usepackage[T1]{fontenc}    
\usepackage{hyperref}       
\usepackage{url}            
\usepackage{booktabs}       
\usepackage{amsfonts}       
\usepackage{nicefrac}       
\usepackage{microtype}      

\title{Prediction Focused Topic Models for Electronic Health Records}

\author{
Jason Ren \thanks{equal contribution} \\
   jason\_ren@college.harvard.edu
  \\
  \And
   Russell Kunes \footnotemark[1] \\
   rk3064@columbia.edu
   \\
   \And
   Finale Doshi-Velez \\
   finale@seas.harvard.edu \\
}

\begin{document}

\maketitle
\begin{abstract}
  
Electronic Health Record (EHR) data can be represented as discrete counts over a high dimensional set of possible procedures, diagnoses, and medications. Supervised topic models present an attractive option for incorporating EHR data as features into a prediction problem: given a patient's record, we estimate a set of latent factors that are predictive of the response variable.  However, existing methods for supervised topic modeling struggle to balance prediction quality and coherence of the latent factors. We introduce a novel approach, the prediction-focused topic model, that uses the supervisory signal to retain only features that improve, or do not hinder, prediction performance.  By removing features with irrelevant signal, the topic model is able to learn task-relevant, interpretable topics.  We demonstrate on a EHR dataset and a movie review dataset that compared to existing approaches, prediction-focused topic models are able to learn much more coherent topics while maintaining competitive predictions.

\end{abstract}

\section{Introduction}

Supervised topic models are often sought to balance prediction quality and interpretability (i.e. \citet{hughes2017predictionanti},  \citet{Kuang2017CrimeTM}). However, standard supervised topic models often learn topics that are not discriminative in target space. Even in the best of cases, these methods must explicitly trade-off between predicting the target and explaining the count data well \citep{hughes2017prediction}. In this work, we focus on one common reason why supervised  topic models fail: documents often contain terms with high occurrence that are irrelevant to the task.  For example, modern electronic health records contain features about all facets of our health, many of which may be nearly orthogonal to any specific task (e.g. predicting risk of some disease). The existence of features irrelevant to the supervised task complicates optimization of the trade-off between prediction quality and explaining the count data, and also renders the topics less interpretable.

To address this issue, we introduce a novel supervised topic model, prediction-focused sLDA (pf-sLDA), that explicitly severs the connection between irrelevant features and the response variable and a corresponding variational inference procedure that enforces our parameter constraints. We demonstrate that pf-sLDA outperforms existing approaches with respect to topic coherence on several data sets, while maintaining competitive prediction quality. The full version of this extended abstract can be viewed at: \url{https://arxiv.org/pdf/1910.05495.pdf}.

\section{Related Work}
\label{sec:related_work}

\textbf{Improving prediction quality in supervised topic models.}
Since the original supervised LDA (sLDA) work of \citet{mcauliffe2008supervised}, many works have incorporated the prediction target into the topic model training process in different ways to improve prediction quality, including power-sLDA \citep{zhang2014supervise}, med-LDA \citep{zhu2012medlda}, BP-sLDA \citep{chen2015end}. \citet{hughes2017prediction} pointed out a number of shortcomings of these previous methods and introduced a new objective that weights a combination of the conditional likelihood and marginal data likelihood: $\lambda \log p(\by|\bw) + \log p(\bw)$. They demonstrated the resulting method, termed prediction-constrained sLDA (pc-sLDA),  achieves better empirical results in optimizing the trade-off between prediction quality and explaining the count data and justify why this is the case.  However, their topics are often polluted by irrelevant terms. The pf-sLDA formulation enjoys analogous theoretical properties but effectively removes irrelevant terms, and thus achieves more coherent topics. 

\textbf{Focusing learned topics.} 
The notion of focusing topics in relevant directions is also present in the unsupervised topic modeling literature.  For example, \citet{Wang2016TargetedTM} focus topics by seeding them with keywords; \citet{Kim2012VariableSF} introduce variable selection for LDA, which models some of the vocabulary as irrelevant.  \cite{Fan2017PriorMS} similarly develop stop-word exclusion schemes. However, these approaches adjust topics based on some general notions of "focus", whereas pf-sLDA removes irrelevant signal for a supervised task to explicitly manage a trade-off between prediction quality and explaining the count data.

\section{Background and Notation}

We briefly describe supervised Latent Dirichlet Allocation (sLDA) \citep{mcauliffe2008supervised}, which our work builds off. sLDA models count data (words) as coming from a mixture of $K$ topics $\{\beta_k\}_{k=1}^K$, where each topic $\beta_k \in \Delta^{|V|-1}$ is a categorical distribution over a vocabulary $V$ of $|V|$ discrete features (words). The count data are represented as a collection of $M$ documents, with each document $\bw_d \in \mathbb{N}^{|V|}$ being a vector of counts over the vocabulary. Each document $d$ is associated with a target $y_d$. Additionally, each document has an associated topic distribution $\theta_d \in \Delta^{K-1}$, which generates both the words and the target.
\section{Prediction Focused Topic Models}
\begin{figure*}[hbt!]
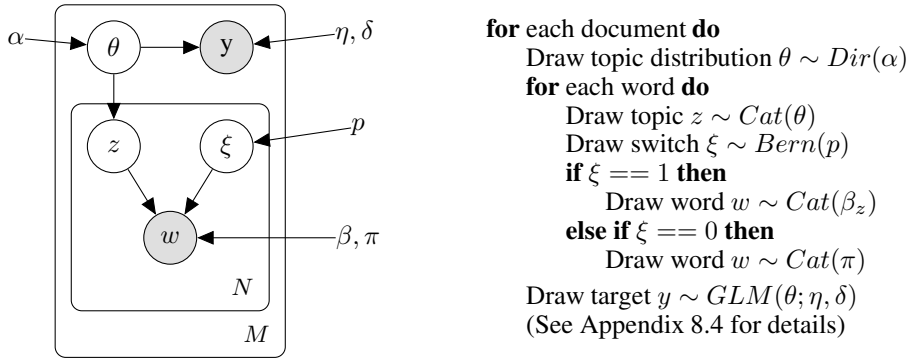

\begin{subfigure}{.5\textwidth}
  \centering
 \tikz{
 \node[obs] (w) {$w$};%
 \node[latent,above=of w,xshift=-0.75cm,fill,yshift=-0.5cm] (z) {$z$}; %
 \node[latent,above=of w,xshift=0.75cm,yshift=-0.5cm] (xi) {$\xi$}; %
 \node[latent, above=of z, yshift=-0.4cm](theta){$\theta$};
 \node[const, above=of z,xshift = -1.3cm](alpha){$\alpha$};
 \node[obs, above= of xi, yshift=-0.4cm](y){y};
 \node[const, above=of xi,xshift = 1.7cm](ed){$\eta, \delta$};
 \node[const, xshift = 2.5cm] (betapi) {$\beta, \pi$};
 \node[const, above = of w, xshift=2.5cm](p) {$p$};
 \plate [inner sep=.2cm,yshift=0cm] {plate1} {(z)(xi)(w)} {$N$}; %
 \plate [inner sep = .2cm,yshift=0cm]
 {plate2} {(theta)(z)(xi)(w)(plate1)} {$M$};
 \edge {z,xi,betapi} {w} 
 \edge {p}{xi}
 \edge {theta}{z,y}
 \edge {alpha}{theta}
 \edge {ed}{y}
 }
  \label{fig:sub1}
\end{subfigure}%
\begin{subfigure}{.5\textwidth}

\centering
\begin{algorithmic}
\For{each document}
\State Draw topic distribution $\theta \sim Dir(\alpha)$
\For{each word}
\State Draw topic $z \sim Cat(\theta)$ 
\State Draw switch $\xi \sim Bern(p)$
\If{$\xi == 1$}
\State Draw word $w \sim Cat(\beta_z)$
\ElsIf{$\xi == 0$}
\State Draw word $w \sim Cat(\pi)$
\EndIf
\EndFor
\State Draw target $y \sim GLM(\theta; \eta, \delta)$ 
\State (See Appendix \ref{sec:impl_deets} for details)
\EndFor
\end{algorithmic}
  
  \label{fig:sub2}
\end{subfigure}
\caption{\textit{Left}: pf-sLDA graphical model. \textit{Right}: pf-sLDA generative process per document}
\label{fig:graphical_models}
\end{figure*}

We now introduce prediction-focused sLDA (pf-sLDA). The fundamental assumption that pf-sLDA builds on is that the vocabulary $V$ can be divided into two disjoint components, one of which is irrelevant to predicting the target variable.  pf-sLDA separates out the words irrelevant to predicting the target, even if they have latent structure, so that the topics can focus on only modelling structure that is relevant to predicting the target. 

\paragraph{Generative Model.}

The pf-sLDA latent variable model has the following components: one channel of pf-sLDA models the count data as coming from a mixture of $K$ topics $\{\beta_k\}_{k=1}^K$, similar to sLDA. The second channel of pf-sLDA models the data as coming from an additional topic $\pi \in \Delta^{|V|-1}$. The target only depends on the first channel, so the second channel acts as an outlet for words irrelevant to predicting the target. We constrain $\beta$ and $\pi$ such that $\beta_k^\top \pi = 0$ for all $k$, such that each word is always either relevant or irrelevant to predicting the target. Which channel a word comes from is determined by its corresponding Bernoulli switch, which has prior $p$. The generative process of pf-sLDA is given in Figure \ref{fig:graphical_models}. In Appendix \ref{sec:lik_lb_deriv}, we prove that a lower bound to the pf-sLDA log likelihood is:
\begin{align}
\log p(\by,\bw) \geq E_{\bxi}[\log p_{\beta}(\by|\bw, \bxi)] + pE_{\btheta}[\log p_{\beta}(\bw|\btheta)] + (1-p)\log p_{\pi}(\bw)
\label{eqn:pfslda_lb}
\end{align}

\textbf{Connection to prediction-constrained models.}
 The lower bound above reveals a connection to the pc-sLDA loss function. Similar connections can be seen in the true likelihood as described in Appendix \ref{sec:pf-sLDA_true_lik_to_pc}, but we use the bound for clarity. The first two terms capture the trade-off between performing the prediction task $E_{\bxi}[\log p_{\beta}(\by|\bw, \bxi)]$ and explaining the words $pE_{\btheta}[\log p_{\beta}(\bw|\btheta)]$, where the switch prior $p$ is used to down-weight the latter task (or emphasize the prediction task). This is analogous to the prediction-constrained objective, but we manage the trade-off through an interpretable model parameter, the switch prior $p$, rather than a more arbitrary Lagrange multiplier $\lambda$.

\section{Inference}

Inference in the pf-sLDA framework corresponds to inference in a graphical model, so advances in Bayesian inference can be applied to solve the inference problem. In this work, we take a variational approach.  Our objective is to maximize the evidence lower bound (full form specified in Appendix \ref{sec:full_elbo}), with the constraint that the relevant topics $\beta$ and additional topics $\pi$ have disjoint support. The key difficulty is that of optimizing over the non-convex set $\{\beta, \pi : \beta^\top \pi = \mathbf{0}\}$. We resolve this with a strategic choice of variational family, which results in a straightforward training procedure that does not require any tuning parameters.
\begin{align*}
    q(\theta, \bz, \bxi| \phi, \varphi, \gamma) &=  \prod_d q(\theta_d | \gamma_d) \prod_n q(\xi_{dn} |\varphi) q (z_{dn} | \phi_{dn}) \\
    \theta_d | \gamma_d &\sim \text{Dir}(\gamma_d),\,
     z_{dn} | \phi_{dn} \sim \text{Cat}(\phi_{dn}),\, 
     \xi_{dn} | \varphi \sim \text{Bern}(\varphi_{w_{dn}})
\end{align*}
The proof of why this variational family enforces our desired constraint is given in Appendix \ref{sec:vfamily_thm_proofs}. To train, we run stochastic gradient descent on the evidence lower bound (ELBO).

\section{Experimental Results}
\subsection{Experimental Set-Up}
\textbf{Metrics.}
We wish to assess prediction quality and interpretability of learned topics. To measure prediction quality, we use RMSE for real targets and AUC for binary targets. To measure interpretability of topics, we use normalized pointwise mutual information coherence, which was shown by \citet{Newman2010AutomaticEO} to be the metric that most consistently and closely matches human judgement in evaluating interpretability of topics. See Appendix \ref{sec:coherence_deets} for coherence calculation details.

\textbf{Baselines.}
The recent work in \cite{hughes2017prediction} demonstrates that pc-sLDA outperforms other supervised topic modeling approaches, so we use pc-sLDA as our main baseline. We also include standard sLDA \citep{mcauliffe2008supervised} for reference.

\textbf{Data Sets}.
 We run our model and baselines on: 
 \begin{itemize}
    \item Pang and Lee's movie review data set \citep{Pang2005SeeingSE}: 5006 movie reviews, with integer ratings from $1$ (worst) to $10$ (best) as targets.
    \item Electronic health records (EHR) of patients with Autism Spectrum Disorder (ASD) \citep{Masood2018APV}: 3804 EHRs, with binary indicator of epilepsy as target.
\end{itemize}
We use a movie review dataset in addition to the ASD dataset, as the movie review dataset is publicly available (so results can be replicated), and results do not require expertise to interpret. (see Appendix \ref{sec:dataset_deets} for details):

\textbf{Implementation details}. Refer to Appendix \ref{sec:impl_deets}

\begin{figure*}[t]
\begin{subfigure}{.5\textwidth}
  \centering
  \begin{tabular}{|p{2cm}||p{1.5cm}|p{1.5cm}|}
 \hline
\multicolumn{3}{|c|}{Pang and Lee Movie Reviews} \\
 \hline
  Model & Coherence & RMSE  \\

\hline
 sLDA & 0.362 (0.101) & 1.682 (0.021) \\
 \hline
pc-sLDA & 1.296 (0.130) & \textbf{1.298} (0.015) \\
 \hline 
 pf-sLDA & \textbf{2.810} (0.092) & 1.305 (0.024) \\
 
 \hline 

\end{tabular}

\label{fig:sub1}
\end{subfigure}
\begin{subfigure}{.5\textwidth}
 \begin{tabular}{|p{2cm}||p{1.5cm}|p{1.5cm}|}
 \hline
  \multicolumn{3}{|c|}{ASD Dataset} \\
  \hline
  Model & Coherence & AUC  \\
\hline
 sLDA & 1.412 (0.113) & 0.590 (0.013)\\
 \hline
pc-sLDA & 2.178 (0.141)  & 0.701 (0.015) \\
\hline
pf-sLDA & \textbf{2.639} (0.091) & \textbf{0.748}  (0.013)\\
 \hline 
 \end{tabular}

\end{subfigure}

\caption{ Mean and (SD) across 5 runs for topic coherence (higher is better) and RMSE (lower is better) or AUC (higher is better) on held-out test sets. Final models were chosen based on a combination of validation coherence and RMSE/AUC. pf-sLDA produces topics with much higher coherence over both data sets, while maintaining similar prediction performance.}
\label{fig:main_res}
\end{figure*}
\begin{table*}[h!]
\centering
\begin{tabular}{|p{1cm}||p{3.7cm}|p{3.7cm}|p{3.7cm}|}
 \hline
\multicolumn{4}{|c|}{Pang and Lee Movie Reviews} \\
 \hline
     & sLDA & pc-sLDA & pf-sLDA\\
 \hline
 High & motion, way, love, performance, \textcolor{lasallegreen}{best}, picture, films, 
character, characters, life &
\textcolor{lasallegreen}{best}, little, time, \textcolor{lasallegreen}{good}, don, picture, year, rated, films
just & 
\textcolor{lasallegreen}{brilliant}, \textcolor{lasallegreen}{rare}, \textcolor{lasallegreen}{perfectly}, true, \textcolor{lasallegreen}{oscar}, documentary, \textcolor{lasallegreen}{wonderful}, 
\textcolor{lasallegreen}{fascinating}, \textcolor{lasallegreen}{perfect}, \textcolor{lasallegreen}{best} \\
\hline \hline

Low & plot, time, \textcolor{red}{bad},  \textcolor{lasallegreen}{funny},  \textcolor{lasallegreen}{good},  \textcolor{lasallegreen}{humor}, little, isn, action & script, year, little,  \textcolor{lasallegreen}{good}, don, look, rated, picture, just, films 
& \textcolor{red}{awful}, \textcolor{red}{stupid}, gags, \textcolor{red}{dumb}, \textcolor{red}{dull}, sequel, \textcolor{red}{flat}, \textcolor{red}{worse}, \textcolor{red}{ridiculous}, \textcolor{red}{bad} \\

\hline \hline
\multicolumn{4}{|c|}{ASD} \\
 \hline
     & sLDA & pc-sLDA & pf-sLDA \\
 \hline
High & Intellect disability  & Infantile cerebral palsy & \textcolor{lasallegreen}{Other convulsions} \\
& Esophageal reflux &  Congenital quadriplegia & Aphasia \\
& Hearing loss & Esophageal reflux & \textcolor{lasallegreen}{Convulsions} \\
& Development delay & fascia Muscle/ligament dis & Central hearing loss \\
& Downs syndrome & Feeding problem &  \textcolor{lasallegreen}{Grand mal status}\\
\hline 
Low & Otitis media & Accommodative esotropia & Autistic disorder \\
& Asthma & Joint pain-ankle & Diabetes Type 1 c0375114\\
& Downs syndrome & Congenital factor VIII & Other symbolic dysfunc \\
& Scoliosis & Fragile X syndrome & Diabetes Type 1 c0375116 \\
& Constipation& Pain in limb & Diabetes Type 2 \\
\hline 

\end{tabular}

\caption{ We list the most probable words in the topics with the highest and lowest regression coefficient for each model for each dataset. Words expected to be in a high regression coefficient topic are listed in green, and words expected to be in a low regression coefficient topic are listed in red. It is clear that the topics learned by pf-sLDA are the most coherent and contain the most words with task relevance.}
\label{tab:topics}
\end{table*}

\subsection{Results}

\textbf{pf-sLDA learns the most coherent topics.} 
Across data sets, pf-sLDA learns the most coherent topics by far (see Figure \ref{fig:main_res}). pc-sLDA improves on topic coherence compared to sLDA, but cannot match the performance of pf-sLDA. Qualitative examination of the topics in Table \ref{tab:topics} supports the claim that the pf-sLDA topics are more coherent, more interpretable, and more focused on the supervised task. 

\textbf{Prediction quality of pf-sLDA remains competitive.}
pf-sLDA produces similar prediction quality compared to pc-sLDA across data sets (see Figure \ref{fig:main_res}). Both pc-sLDA and pf-sLDA outperform sLDA in prediction quality. In the best performing models of pf-sLDA for both data sets, generally between $10\%$ and $20\%$ of the words were considered relevant.



 \section{Conclusion}
 In this paper, we introduced prediction-focused supervised LDA, whose vocabulary selection procedure improves topic coherence of supervised topic models while maintaining competitive prediction quality. Future work could include establishing additional theoretical properties of the pf-sLDA variable selection procedure, and applying our trick of managing trade-offs within a graphical model for variable selection in other generative models.

\bibliography{refs} 
\bibliographystyle{ieeetr}

\newpage
\section{Appendix}

\subsection{Variational Family}
\label{sec:vfamily_thm_proofs}
We show how our choice of variational family incorporates our desired constraint in the model parameters. The constraint we wish to satisfy is $\beta^\top \pi = \bold{0}$. Our choice of variational family is:
\begin{align*}
    q(\theta, \bz, \bxi| \phi, \varphi, \gamma) &=  \prod_d q(\theta_d | \gamma_d) \prod_n q(\xi_{dn} |\varphi) q (z_{dn} | \phi_{dn}) \\
    \theta_d | \gamma_d &\sim \text{Dir}(\gamma_d) \\
     z_{dn} | \phi_{dn} &\sim \text{Cat}(\phi_{dn}) \\ 
     \xi_{dn} | \varphi &\sim \text{Bern}(\varphi_{w_{dn}}) \\ 
\end{align*}
where $d$ indexes over the documents and $n$ indexes over the words in each document. We first propose two theorems relating to the model.

\begin{thm}
Suppose that the channel switches $\xi_d$ and the document topic distribution $\theta_d$ are conditionally independent in the posterior for all documents, then $\beta$ and $\pi$ have disjoint supports over the vocabulary.
\end{thm}

\begin{proof} For simplicity of notation, we assume a single document and hence drop the subscripts on $\xi_d$ and $\theta_d$. All of the arguments are the same in the multi-document case.
If $\xi$ and $\theta$ are conditionally independent in the posterior, then we can factor the posterior as follows: $p(\xi, \theta | \bw, y) = p(\xi | \bw, y) p(\theta | \bw, y)$. We expand out the posterior:
\begin{align*}
    p(\xi, \theta | \bw, y) &\propto p(\xi) p(\theta) p(\bw, y | \theta, \xi) \\
    &\propto p(\xi) p(\theta) p(y | \theta) \prod_n p_\beta(w_n | \theta) ^{\xi_n} p_\pi(w_n)^{1 - \xi_n} \\
    &= f(\theta) g(\xi) \prod_n p_\beta(w_n | \theta) ^{\xi_n}
\end{align*}

for some functions $f$ and $g$. Thus we see that we must have that $ \prod_n p_\beta(w_n | \theta) ^{\xi_n}$  factors into some $r(\theta) s(\xi)$. We expand $\prod_n p_\beta(w_n | \theta) ^{\xi_n}$:
\begin{align*}
    p_\beta(w_n | \theta)^{\xi_n} &= \left( \sum_k \beta_{k, w_n} \theta_k\right)^{\xi_n} \\
    &= I(\xi_n = 0) + I(\xi_n = 1) \left( \sum_k \beta_{k, w_n} \theta_k\right) 
\end{align*}
So that we can express the product as:
\begin{align*}
    \prod_n p_\beta(w_n | \theta)^{\xi_n} &= \prod_n \left\{ I(\xi_n = 0) + I(\xi_n = 1) \left( \sum_k \beta_{k, w_n} \theta_k\right)\right\}
\end{align*}
In order to further simplify, let $\beta_0 = \{n : \sum_k \beta_{k, w_n} = 0\}$ and $\beta_{>} = \{n : \sum_k \beta_{k, w_n} > 0\}$. In other words $\beta_0$ is the set of $n$ such that the word $w_n$ is not supported by $\beta$, and $\beta_{>}$ is the set of $n$ such that the word $w_n$ is supported by $\beta$.

We can rewrite the above as:
\begin{align*}
    \prod_n p & _\beta(w_n | \theta)^{\xi_n} = \left(\prod_{n \in \beta_0} I(\xi_n = 0) \right) \left(\prod_{n \in \beta_>} \left\{ I(\xi_n = 0) +  I(\xi_n = 1)  \sum_k \beta_{k, w_n} \theta_k \right\} \right)
\end{align*}

Thus, we see that we can factor $\prod_n p_\beta(w_n | \theta) ^{\xi_n}$ as a function of $\theta$ and $\xi$ into the form $r(\theta) s(\xi)$ only if $\xi_n = 0$ or $\xi_n =1$ with probability $1$. We can check that this implies $\beta_k^\top \pi = 0$ for each $k$ by the result of Theorem $2$. 
\end{proof}

\begin{thm}
$\beta^\top \pi  = 0$ if and only if there exists a $\xi^*$ s.t.  $p(\xi^* |\bw, y) = 1$ 
\end{thm}

\begin{proof}

\begin{enumerate}
    \item Assume $\beta^\top \pi  = 0$. Then, conditional on $w_n$, $\xi_n = 1$ with probability 1 if $\pi_{w_n} = 0$ and  $\xi_n = 0$ with probability 1 if $\pi_{w_n} > 0$. So we have $p(\xi^* |\bw, y) = 1$ for the $\xi^*$ corresponding to $\bw$ as described before.
    \item Assume there exists a $\xi^*$ s.t.  $p(\xi^* |\bw, y) = 1$.
    
    Then we have:
    \begin{align*}
        p(\xi^* | \bw, y) = \frac{p(\bw, y | \xi^*) p(\xi^*)}{\sum_{\xi} p(\bw, y | \xi) p(\xi)} = 1 \\
        p(\bw, y | \xi^*) p(\xi^*) = \sum_{\xi} p(\bw, y | \xi) p(\xi)
    \end{align*}
    
    This implies $p(\bw, y | \xi) p(\xi) = 0 \ \forall\  \xi \neq \xi^*$, which implies $p(\bw, y | \xi) = 0 \ \forall\ \xi \neq \xi^*$
    
    Then we have:
    \begin{align*}
        p(\bw, y | \xi) &= p(y |\bw, \xi) p(\bw| \xi) \\ 
        &= \left(\int_\theta p(y | \theta) p(\theta | \bw, \xi) d\theta \right)\left( \int_\theta p(\bw | \theta, \xi) p(\theta) d\theta \right)
    \end{align*}
    
    The first term will be greater than 0 because $y|\theta$ is distributed Normal. We focus on the second term.
    
    \begin{align*}
         \int_\theta p(\bw | \theta, \xi) p(\theta) d\theta = \int_\theta p(\theta) \prod_n p_\beta(w_n | \theta) ^{\xi_n} p_\pi(w_n)^{1 - \xi_n}d\theta
    \end{align*}
    
    Let $X$ be the set of $\xi$ that differ from $\xi^*$ in one and only one position, i.e. $\xi_n = \xi^*_n$ for all $n \in \{1,\dots N\}\setminus \{i\}$ and $\xi_i \neq \xi^*_i$. For each $\xi \in X$, $\int_\theta p(\theta) \prod_n p_\beta(w_n | \theta) ^{\xi_n} p_\pi(w_n)^{1 - \xi_n} = 0$. Since all functions in the integrand are non-negative and continuous, $p_\beta(w_n | \theta) ^{\xi_n} p_\pi(w_n)^{1 - \xi_n} = 0$ for the unique $i$ with $\xi_i \neq \xi_i^*$. Since this holds for every element of $X$, we must have that $p_\beta(w_n | \theta) = 0$ for all $\xi_n = 0$ and $p_\pi(w_n) = 0$ for all $\xi_n = 1$, proving $\beta$ and $\pi$ are disjoint, provided the minor assumption that all words in the vocabulary $w_n$ are observed in the data. In practice all words are observed in the vocabulary because we choose the vocabulary based on the training set.
\end{enumerate}
\end{proof}

Theorems 1 and 2 tell us that if the posterior distribution of the channel switches $\bxi_d$ is independent of the posterior distribution  of the document topic distribution $\theta_d$ for all documents, then the true relevant topics $\beta$ and additional topic $\pi$ must have disjoint support, and moreover the posterior of the channel switches $\bxi$ is a point mass. This suggests that to enforce that $\beta$ and $\pi$ are disjoint, we should choose the variational family such that $\bxi$ and $\btheta$ are independent. 

If $\bxi$ and $\btheta$ are conditionally independent in the posterior, then the posterior can factor as $p(\bxi, \btheta|\by, \bw) = p(\bxi |\by,\bw) p(\btheta | \by, \bw)$. In this case, the posterior for the channel switch of the $n$th word in document $d$, $\xi_{dn}$, has no dependence $d$, which can be seen directly from the graphical model. Thus, choosing $q(\bxi | \varphi)$ to have no dependence on document naturally pushes our assumption into the variational posterior.

Our choices for the variational distributions for $\btheta$ and $\bz$ match those of \citet{mcauliffe2008supervised}. We choose $q(\xi_{dn}|\varphi_{w_{dn}})$ to be a Bernoulli probability mass function with parameter $\varphi_{w_{dn}}$ indexed only by the word $w_{dn}$. This distribution acts as a relaxation of a true point mass posterior, allowing us to use gradient information to optimize over $[0,1]$ rather than directly over $\{0,1\}$.  Moreover, this parameterization allows us to naturally use the variational parameter $\varphi$ as a feature selector; low estimated values of $\varphi$ indicate irrelevant words, while high values of $\varphi$ indicate relevant words. 

\subsection{ELBO (per doc)}
\label{sec:full_elbo}
Let $\Lambda= \{ \alpha, \beta, \eta, \delta, \pi, p \} $. Omitting variational parameters for simplicity:
\begin{align*}
    \log p(\bw,\by | \Lambda) &= \log \int_\theta \sum_z \sum_\xi p(\theta, \bz, \bxi, \bw, \by | \Lambda) d\theta \\
    &= \log E_q \left(\frac{p(\theta, \bz, \bxi, \bw, \by | \Lambda)} {q(\theta, \bz, \bxi)}\right) \\
    &\geq E_q[\log p(\theta, \bz, \bxi, \bw, \by)] - E_q[q(\theta, \bz, \bxi)]
\end{align*}

Let $ELBO = E_q[\log p(\theta, \bz, \bxi, \bw, \by | \Lambda)] - E_q[q(\theta, \bz, \bxi)]$

Expanding this:
\begin{align*}
    ELBO &= E_q[\log p(\theta | \alpha)] + E_q[\log p(\bz|\theta)] + E_q[\log p(y|\theta, \eta, \delta)]\\
    &+ E_q[\log p(\bxi| p)] + E_q[\log p(\bw | \bz, \beta, \bxi, \pi)] \\
    &- E_q[\log q(\theta | \gamma)] - E_q[\log q(\bz | \phi)] - E_q[\log q(\bxi | \varphi)]
\end{align*}
The distributions of each of the variables under the generative model are:
\begin{align*}
    &\theta_d \sim \text{Dirichlet}(\alpha)\\
    &z_{dn} |\theta_d \sim \text{Categorical}(\theta_d)\\
    &\xi_{dn} \sim \text{Bernoulli}(p)\\
    &w_{dn}|z_{dn},\xi_{dn} = 1 \sim \text{categorical}(\beta_{z_{dn}})\\
    &w_{dn}|z_{dn},\xi_{dn} = 0 \sim \text{Categorical}(\pi)\\
    &y_d|\theta_d \sim \text{GLM}(\theta ; \eta, \delta)
\end{align*}
Under the variational posterior, we use the following distributions:
\begin{align*}
    &\theta_d \sim \text{Dirichlet}(\gamma_d)\\
    &z_{dn} \sim \text{Categorical}(\phi_{dn})\\
    &\xi_{dn} \sim \text{Bernoulli}(\varphi_{w_{dn}})
\end{align*}
    
This leads to the following ELBO terms:
\begin{align*}
    E_q[\log p(\theta | \alpha)] &=  \log \Gamma\left(\sum_k \alpha_k\right) - \sum_k \log \Gamma(\alpha_k) 
    + \sum_k(\alpha_k - 1)E_q[\log \theta_k] \\
    E_q[\log p(\bz | \theta)] &= \sum_n \sum_k \phi_{nk}E_q[\log \theta_k] \\
    E_q[\log p(\bw | \bz, \beta, \bxi, \pi)] &= \sum_n \left(\sum_v w_{nv} \varphi_v \right) *\left(\sum_k \sum_v \phi_{nk} w_{nv} \log \beta_{kv}\right) \\
    &+\left(1 - \left(\sum_v w_{nv} \varphi_v \right)\right) \left(\sum_v w_{nv} \log \pi_v \right) \\
    E_q[\log p(\bxi| p)] &= \sum_n \left(\sum_v w_{nv} \varphi_v \right) \log p  + \left(1 - \left(\sum_v w_{nv} \varphi_v \right)\right)\log (1-p) \\
    E_q[q(\theta | \gamma)] &= \log \Gamma\left(\sum_k \gamma_k\right) - \sum_k \log \Gamma(\gamma_k) + \sum_k(\gamma_k - 1)E_q[\log \theta_k] \\
    E_q[q(\bz | \phi)] &= \sum_n \sum_k \phi_{nk} \log \phi_{nk} \\
    E_q[q(\bxi| \varphi)] &= \sum_n \left(\sum_v w_{nv} \varphi_v \right) \log \left(\sum_v w_{nv} \varphi_v \right) \\
    &+ \left(1 - \left(\sum_v w_{nv} \varphi_v \right)\right) \log \left(1 - \left(\sum_v w_{nv} \varphi_v \right)\right) \\
    E_q[\log p(y|\theta, \eta, \delta)] &= \frac{1}{2}\log 2\pi \delta - \frac{1}{2\delta}\left(y^2 -2y\eta^\top E_q[\theta] + \eta^\top E_q[\theta \theta^\top]\eta \right)
\end{align*}

Other useful terms:
\begin{align*}
E_q[\log \theta_k] &= \Psi(\gamma_k) - \Psi\left(\sum_{j=1}^K\gamma_{j}\right) \\
\bar{Z} &:= \frac {\sum_n \xi_n z_n} {\sum_n \xi_n} \in \mathbb{R}^K \\
E_q[\theta] &= \frac {\gamma} {\gamma^\top \bold{1}} \\
\gamma_0 &:= \sum_k \gamma_k \\
\tilde{\gamma}_j &:= \frac {\gamma_j} {\sum_k \gamma_k} \\
E_q[\theta \theta^\top]_{ij} &= \frac{\tilde{\gamma}_i (\delta(i,j) - \tilde{\gamma}_j)}{\gamma_0 + 1} + \tilde{\gamma}_i\tilde{\gamma}_j
\end{align*}

\subsection{Lower Bounds on the Log Likelihood}
\label{sec:lik_lb_deriv}
Remark that the likelihood for the words of one document can be written as follows:     

\[
p(\bw) = \int_{\theta}d\theta p(\theta|\alpha) \left\{\prod_{n=1}^N [p*p_{\beta}(w_n|\theta)+ (1-p)p_{\pi}(w_n)] \right\} 
\]

We would like to derive a lower bound to the joint log likelihood $p(y,\bw)$ of one document that resembles the prediction constrained log likelihood since they exhibit similar empirical behavior. Write $p(y,\bw)$ as $E_{\xi}[p(y|\bw,\xi)p(\bw|\xi)]$ and apply Jensen's inequality:

\[
\log p(y,\bw) \geq E_{\xi}[\log p(y|\bw,\xi)] + E_{\xi}[\log p(\bw|\xi)]
\]
Focusing on the second term we have:
\[
\log p(\bw|\xi) = \log \int_{\theta} d\theta p(\theta|\alpha) \prod_{n=1}^N p_{\beta}(w_n|\theta)^{\xi_n}p_{\pi}(w_n)^{1-\xi_n}
\]
Applying Jensen's inequality again to push the $\log$ further inside the integrals:
\[
\log p(\bw|\xi) \geq \int_{\theta}d\theta p(\theta|\alpha) \Bigg\{ \sum_{i=1}^N\xi_n \log p_{\beta}(w_n|\theta) + \sum_{n=1}^N(1-\xi_n)\log p_{\pi}(w_n) \Bigg\}
\]

Note that $\theta$ and $\xi$ are independent, so we have:

\[
\log p(y,\bw) \geq E[\log p(y|\bw,\xi)] + E\left[\sum_{i=1}^N\xi_n \log p_{\beta}(w_n|\theta) + \sum_{n=1}^N(1-\xi_n)\log p_{\pi}(w_n)\right]
\]

where the expectation is taken over the $\xi$ and $\theta$ priors. This gives the final bound:

\[
\log p(y,\bw) \geq E[\log p_{\beta}(y|W_1(\xi))] + pE[\log p_{\beta}(\bw|\theta)] + (1-p)\log p_{\pi}(\bw)
\]

We have used the substitution: $p(y|\bw,\xi) = p_{\beta}(y|W_1(\xi))$. Conditioning on $\xi$, $y$ is independent from the set of $w_n$ with $\xi_n = 0$, so we denote $W_1(\xi)$ as the set of $w_n$ with $\xi_n = 1$. It is also clear that $p(y|W_1(\xi), \xi) = p_\beta(y|W_1(\xi))$. By linearity of expectation, this bound can easily be extended to all documents.

Note that this bound is undefined on the constrained parameter space: $\beta^\top \pi = 0$; if $p\neq 0$ and $p\neq 1$. This is clear because $\log p_{\pi}(\bw)$ or $\log p_{\beta}(w_n|\theta)$ is undefined with probability 1. We can also see this directly, since $p(y,\bw|\xi)$ is non-zero for exactly one value of $\xi$ so $E[\log p(y,\bw|\xi)]$ is clearly undefined. We derive a tighter bound for this particular case as follows. Define $\xi^*(\pi, \beta, \bw)$ as the unique $\xi$ such that $p(\bw|\xi)$ is non-zero. We can write $p(y,\bw) = p(y,W | \xi^*(\pi, \beta, \bw)) p(\xi^*(\pi, \beta, \bw))$. For simplicity, I use the notation $\xi^*$ but keep in mind that it's value is determined by $\beta$, $\pi$ and $\bw$. Also remark that the posterior of $\xi$ is a point mass as $\xi^*$. If we repeat the analysis above we get the bound:
\[
\log p(y,\bw) \geq p_\beta(y| W_1(\xi^*)) + E\left[\sum_{n=1}^N\xi^*p_\beta(w_n|\theta)\right] + \sum_{n=1}^N(1-\xi^*)\log p_{\pi}(w_n) + p(\xi^*)
\]
which is to be optimized over $\beta$ and $\pi$. Note that the $p(\xi^*)$ term is necessary because of its dependence on $\beta$ and $\pi$.  
 Comparing this objective to our ELBO, we make a number of points. The true posterior is $\xi^*$ which would ordinarily require a combinatorial optimization to estimate; however we introduce the continuous variational approximation $\xi \sim Bern(\varphi)$. Note that the true posterior is a special case of our variational posterior (when $\varphi =1$ or $\varphi = 0$). Since the parameterization is differentiable, it allows us to estimate $\xi^*$ via gradient descent. Moreover, the parameterization encourages $\beta$ and $\pi$ to be disjoint without explicitly searching over the constrained space. Empirically, the estimated set of $\varphi$ are correct in simulations, and correct given the learned $\beta$ and $\pi$ on real data examples.

\subsection{Implementation details}
\label{sec:impl_deets}
Code base: \url{https://github.com/jasonren12/PredictionFocusedTopicModel}

In general, we treat $\alpha$ (the prior for the document topic distribution) as fixed (to a vector of ones). We tune pc-SLDA using \cite{hughes2017prediction}'s code base, which does a small grid search over relevant parameters. We tune sLDA and pf-sLDA using our own implementation and SGD. $\beta$ and $\pi$ are initialized with small, random (exponential) noise to break symmetry. We optimize using ADAM with initial step size 0.025.

We model real targets as coming from $N(\eta^\top \theta, \delta)$ and binary targets as coming from Bern$(\sigma(\eta^\top  \theta))$

\subsection{pf-sLDA likelihood and prediction constrained training.}
\label{sec:pf-sLDA_true_lik_to_pc}
The pf-sLDA marginal likelihood for one document and target can be written as:
\begin{align*}
    p(\bw, y) &= p(y | \bw) \int_\theta \sum_{\xi} p(\bw, \theta, \xi)\\ 
    &= p(y | \bw) \int_\theta p(\theta | \alpha) \prod_n\big\{ p * p_\beta(w_n | \theta, \xi_n = 1, \beta) + (1-p) * p_\pi(w_n | \xi_n = 0, \pi)\big\}
\end{align*}
where $n$ indexes over the words in the document. We see there still exist the $p(y|\bw)$ and $p * p_\beta(\bw)$ that are analogous to the prediction constrained objective, though the precise form is not as clear.

\subsection{Coherence details}
\label{sec:coherence_deets}
We calculate coherence for each topic by taking the top 50 most likely words for the topic, calculating the pointwise mutual information for each possible pair, and averaging. These terms are defined below.
\begin{align*}
    \text{coherence} &= \frac{1}{N (N-1)} \sum_{w_i,w_j \in \text{TopN}} \text{pmi}(w_i, w_j) \\ 
    \text{pmi}(w_i, w_j) &= \log \frac {p(w_i) p(w_j)}{p(w_i, w_j)} \\ 
    p(w_i) &= \frac{\sum_d I(w_i \in \text{doc d})} M \\
    p(w_i, w_j) &= \frac{\sum_d I(w_i \text{ and } w_j \in \text{doc d})} M
\end{align*} 
where $M$ is the total number of documents and $N=50$ is the number of top words in a topic.

The final coherence we report for a model is the average of all the topic coherences.

\subsection{Data set details}
\label{sec:dataset_deets}
\begin{itemize}
     \item Pang and Lee's movie review data set \citep{Pang2005SeeingSE}: There are 5006 documents. Each document represents a movie review, and the documents are stored as bag of words and split into 3754/626/626 for train/val/test. After removing stop words and words appearing in more than $50\% $ of the reviews or less than $10$ reviews, we get $|V| = 4596$. The target is an integer rating from $1$ (worst) to $10$ (best).
     \item Electronic health records (EHR) data set of patients with Autism Spectrum Disorder (ASD), introduced in \cite{Masood2018APV}: There are 3804 documents. Each document represents the EHR of one patient, and the features are possible diagnoses. The documents are split into 3423/381 for train/val, with $|V| = 3600$. The target is a binary indicator of presence of epilepsy.
 \end{itemize}

\end{document}